%% file: main.tex
\pdfoutput=1
\documentclass[sigconf]{acmart}

\usepackage{algorithm}
\usepackage{algorithmic}
\usepackage{multirow}
\usepackage{booktabs}
\usepackage{graphicx}
\usepackage{xcolor}
\usepackage{amsmath,amsfonts,amsthm}
\usepackage{mathtools}

\newtheorem{theorem}{Theorem}[section]

\newtheorem{lemma}[theorem]{Lemma}
\AtBeginDocument{%
  }

\copyrightyear{2024}
\acmYear{2024}
\setcopyright{acmlicensed}\acmConference[KDD '24]{Proceedings of the 30th ACM SIGKDD Conference on Knowledge Discovery and Data Mining}{August 25--29, 2024}{Barcelona, Spain}
\acmBooktitle{Proceedings of the 30th ACM SIGKDD Conference on Knowledge Discovery and Data Mining (KDD '24), August 25--29, 2024, Barcelona, Spain}
\acmDOI{10.1145/3637528.3671728}
\acmISBN{979-8-4007-0490-1/24/08}

\begin{document}

\title{Layer-wise Adaptive Gradient Norm Penalizing Method for Efficient and Accurate Deep Learning}

\author{Sunwoo Lee}
\orcid{0000-0001-6334-3068}
\affiliation{%
  \institution{Inha University}
  \department{Department of Computer Engineering}
  \streetaddress{100, Inha-ro, Michuhol-gu,}
  \city{Incheon}
  \country{Republic of Korea}
  \postcode{22212}
}
\email{sunwool@inha.ac.kr}

\renewcommand{\shortauthors}{Sunwoo Lee}

\begin{abstract}
Sharpness-aware minimization (SAM) is known to improve the generalization performance of neural networks. However, it is not widely used in real-world applications yet due to its expensive model perturbation cost. A few variants of SAM have been proposed to tackle such an issue, but they commonly do not alleviate the cost noticeably. In this paper, we propose a lightweight layer-wise gradient norm penalizing method that tackles the expensive computational cost of SAM while maintaining its superior generalization performance. Our study empirically proves that the gradient norm of the whole model can be effectively suppressed by penalizing the gradient norm of only a few critical layers. We also theoretically show that such a partial model perturbation does not harm the convergence rate of SAM, allowing them to be safely adapted in real-world applications. To demonstrate the efficacy of the proposed method, we perform extensive experiments comparing the proposed method to mini-batch SGD and the conventional SAM using representative computer vision and language modeling benchmarks.
\end{abstract}

\ccsdesc[500]{Computing methodologies~Neural networks}

\keywords{Deep Learning, Sharpness-Aware Minimization, Layer-wise}


\maketitle

\input{1_intro}
\input{2_related}

\input{3_method}

\input{4_eval}
\input{5_conclusion}

\section*{Acknowledgments}
This work was partly supported by Institute of Information \& communications Technology Planning \& Evaluation (IITP) grant funded by the Korea government(MSIT) (No.RS-2022-00155915, Artificial Intelligence Convergence Innovation Human Resources Development (Inha University)).
This research was supported by the MSIT(Ministry of Science, ICT), Korea, under the National Program for Excellence in SW), supervised by the IITP(Institute of Information \& communications Technology Planing \& Evaluation) in 2024 (2022-0-01127).
This work was supported by the National Research Foundation of Korea(NRF) grant funded by the Korea government(MSIT) (No. RS-2023-00279003).
This work was supported by INHA UNIVERSITY Research Grant.

\bibliographystyle{ACM-Reference-Format}
\bibliography{mybib}

\appendix
\input{appendix}

\end{document}

%% file: 1_intro.tex
\section {Introduction} \label{sec:intro}
Sharpness-aware minimization (SAM) \cite{foret2020sharpness} improves the generalization performance of Machine Learning models by penalizing the norm of gradients.
The reduced gradient norm results in leading the model to a flat region of the parameter space \cite{zhao2022penalizing}.
It is already widely known that the convergence to a flat region helps improve the generalization performance \cite{hochreiter1997flat,keskar2016large,wen2018smoothout,chaudhari2019entropy,baldassi2020shaping,yao2020pyhessian,petzka2021relative,lee2023achieving}.

Despite its superior generalization performance, SAM is not popularly used in real-world ML applications yet due to the expensive computational cost.
Specifically, it calculates the gradients twice, one for the model perturbation (gradient ascent) and the other for the model update (gradient descent).
Thus, it is crucial to develop an efficient gradient norm penalizing method to allow real-world applications to adopt SAM and enjoy improved generalization performance.

A few recent studies tackle the expensive computational cost issue in SAM.
Du et al. perturb a random subset of the model parameters to reduce the extra computational cost \cite{du2021efficient}.
While this work presents promising results, they evaluated the performance only using the parameter-wise perturbation probability of $0.5 \sim 0.6$.
Thus, it is not clear whether the generalization performance can be maintained when the fraction of the perturbed model is small.
Liu et al. directly reduce the gradient ascent cost by reusing the gradient ascent for multiple iterations \cite{liu2022towards}.
However, the gradient can be safely reused only when the dataset gives stable gradients for many iterations.
That is, the effectiveness of the periodic gradient ascent computations depends on the given data characteristics, making it less practical.
Mi et al. find a critical subset of the model parameters based on Fisher information and perturb them only\cite{mi2022make}.
Although this work provides critical insight into the partial model perturbation approach, they first calculate the full gradient ascent and then apply a sparse filter, having a limited performance gain.

In this paper, we propose a novel layer-wise adaptive gradient norm penalizing method that tackles the expensive model perturbation cost issue in SAM while maintaining its superior generalization performance.
Instead of perturbing the whole model parameters, we propose to selectively perturb only a few network layers that strongly affect the model's generalization performance.
Our empirical study shows that a few layers consistently have a larger gradient norm than other layers during training.
Interestingly, many networks tend to have a large gradient norm at the output-side layers.
By leading those layers to a flat region in the parameter space, it is expected that the entire model will have a much smaller gradient norm, thereby improving its generalization performance.
Our theoretical analysis also shows that such a layer-wise method does not harm the convergence rate regardless of which and how many layers are selected to be perturbed.

Our method is differentiated from the existing SAM variants in the sense that it exploits the internal data characteristics of neural networks to make the best trade-off between the generalization performance and the extra computational cost, rather than just randomly perturbing a subset of model parameters.
First, we empirically prove that only a subset of the model parameters strongly affect the model's generalization performance.
The proposed layer-wise method effectively suppresses the gradient norm of the whole model parameters and it results in achieving the improved generalization performance.
Second, we also show that the gradient norm of the entire model can be reduced at a marginal extra computational cost by the layer-wise adaptive model perturbation.
Our proposed method entirely skips the gradient computation at the layers that are not perturbed.
This approach not only mitigates the extra computational cost but also eases the implementation complexity, making SAM more practical.

To demonstrate the efficacy of our layer-wise gradient penalizing method, we compare it to the SOTA variants of SAM as well as the conventional mini-batch SGD.
Our experimental results show that the full model gradient norm can be effectively penalized by the proposed layer-wise method, and it results in achieving similar accuracy to the full model perturbation method (SAM).
In addition, the proposed method pushes the training throughput towards that of the conventional mini-batch SGD.
For example, when training ResNet20 on CIFAR-10, our method achieves a similar validation accuracy to SAM while it has just $11\%$ reduced throughput as compared to the mini-batch SGD.
Our empirical and theoretical studies demonstrate that the gradient norm penalizing method can be efficiently employed in real-world applications.

%% file: 2_related.tex
\section {Related Work} \label{sec:related}
Foret et al. proposed Sharpness-Aware Minimization, an optimization method that leads the model toward a flat region in the parameter space so that it achieves better generalization performance \cite{foret2020sharpness}.
Zheng et al. also concurrently proposed a similar method in \cite{zheng2021regularizing}.
After its superior generalization performance had been observed, several researchers analyzed the theoretical performance to better understand its behaviors.
Andriushchenko et al. presented a thorough theoretical analysis of the SAM's performance in \cite{andriushchenko2022towards}.
Bartlett et al. analyzed the dynamics of SAM for convex quadratic functions \cite{bartlett2023dynamics}.
Zhao et al. introduced a generalized version of SAM, which allows finer-grained model perturbation with another hyper-parameter \cite{zhao2022penalizing}.
Zhoe et al. analyzed the behaviors of SAM for problems with class imbalance issues \cite{zhou2023imbsam}.
Recently, Caldarola et al. applied SAM to Federated Learning, showing a practical use-case of it \cite{caldarola2022improving}.

Recently, several variants of SAM have been proposed.
Kwon et al. developed adaptive SAM, which alleviates the sensitivity to parameter re-scaling by adaptively adjusting the model perturbation terms \cite{kwon2021asam}.
Zhang et al. designed Gradient Norm-Aware Minimization which boosts up the generalization performance by utilizing approximated second-order information \cite{zhang2023gradient}.
Zhuang et al. proposed an alternative way of leading the model to a flat region in the parameter space by using a customized metric, surrogate gap \cite{zhuang2022surrogate}.
Mi et al. found critical parameters based on Fisher information and dropped the gradient of the rest of the parameters in the gradient ascent step to stabilize the training \cite{mi2022make}.
Jiang et al. accelerated SAM by applying adaptive learning rate adjustment and theoretically analyzed its performance \cite{jiang2023adaptive}.
However, these studies do not discuss the expensive model perturbation cost in SAM.

Liu et al. efficiently perturbed model parameters by re-using the gradients multiple times \cite{liu2022towards}.
Du et al. perturbed a random subset of the model parameters to save the model perturbation cost \cite{du2021efficient}.
Qu et al. applied SAM to Federated Learning and demonstrated that the client-side SAM can improve the generalization performance of the global model \cite{qu2022generalized}.
Unfortunately, these methods either do not noticeably reduce the computational cost or reduce the computational cost only when the dynamics of the stochastic gradients remain stable.

%% file: 3_method.tex
\section {Layer-Wise Adaptive Gradient Norm Penalizing Method} \label{sec:method}

\subsection {Problem Setting}
We consider an empirical risk minimization problem with a special loss function that penalizes the gradient norm:
\begin{align}
    L(w) := L_S(w) + \lambda \| \nabla L_S(w) \|_p. \label{eq:loss}
\end{align}
In the above equation, $\| \cdot \|_p$ denotes the $L_p$ norm, $N$ is the number of training samples, and $L_S(w)$ is the empirical loss function $\frac{1}{N} \sum_{i=1}^{N} l(w, \xi_i)$ where $\xi_i$ is a training data sample $i$.
The coefficient $\lambda$ determines how strongly the gradient norm is penalized.
We focus on the $L_2$ norm in this work.

Several previous works already showed how to calculate the gradient of (\ref{eq:loss}) \cite{andriushchenko2022towards,zhao2022penalizing,liu2022towards}.
As shown in \cite{zhao2022penalizing}, the gradient of (\ref{eq:loss}) is calculated as follows.
\begin{align}
    \nabla L(w) = (1 - \alpha) \nabla L_S(w) + \alpha \nabla L_S \left( w + \rho \frac{\nabla L_S(w)}{\| \nabla L_S(w) \|} \right), \label{eq:grads}
\end{align}
where $\rho$ is a small constant used to approximate the second term of the right-hand side and $\alpha = \frac{\lambda}{\rho}$.
Note that \lq{}sharpness-aware minimization\rq{}, SAM, is a special case where $\alpha = 1$.
A majority of existing works simplify (\ref{eq:grads}) by omitting the gradient normalization to provide the convergence guarantee \cite{andriushchenko2022towards,mi2022make,jiang2023adaptive,sun2023adasam} as follows.
\begin{align}
    \nabla L(w) = (1 - \alpha) \nabla L_S(w) + \alpha \nabla L_S \left( w + r \nabla L_S(w) \right), \label{eq:grads2}
\end{align}
where $r$ is a small constant.
Following such a convention, we consider the simplified version of the gradients in the rest of the paper.

\subsection {Layer-Wise Gradient Norm Penalizing Method}

To obtain $\nabla L(w)$ shown in (\ref{eq:grads}), the gradient should be calculated twice, doubling the computational cost of training.
We propose a layer-wise gradient norm penalizing method to tackle this critical issue.
The layer-wise approach has been recently investigated in a several different contexts~\cite{schulz2012deep,you2018imagenet,ma2022layer,lee2023partial,lee2023layer}.
Instead of penalizing the norm of the total gradients as (\ref{eq:loss}), we penalize the norm of only a few chosen layers using the loss function is as follows.
\begin{align}
    L(w) := L_S(w) + \lambda \| \nabla L_S(w') \|_p, \label{eq:loss2} \\
    w' \subset w = \{w_1, \cdots, w_L \} \label{eq:subset},
\end{align}
where $L$ is the number of network layers.
The $w$ is the full model that contains $L$ layers and $w_i$ is the $i^{th}$ layer.
Consequently, the layer-wise gradient of (\ref{eq:loss2}) for the layers in $w'$ becomes as follows.
\begin{align}
    \nabla L(w_i) &= (1 - \alpha) \nabla L_S(w_i) + \alpha \nabla L_S \left( w_i + r \nabla L_S(w_i) \right). \label{eq:grads3}
\end{align}
The layer-wise gradient of all the other layers is just simply $\nabla L(w_i) = \nabla L_S(w_i)$.
Unless every layer equally contributes to the models' generalization performance, there should exist subsets of layers $w'$ that give better generalization performance than other subsets.
Our goal is to find such subsets during training and apply the gradient norm penalizing method only to them so that the model achieves good generalization performance at a marginal model perturbation cost.

\subsection {Convergence Analysis} \label{sec:theory}
Before we discuss how to determine which layers to perturb, we first provide a performance guarantee of the layer-wise gradient penalizing method.
We analyze the convergence properties of the proposed method for smooth and non-convex optimization problems.
Our analysis assumes the following conditions on the loss function $\ell_i(w) := \ell (w, \xi_i), i \in \{1, \cdots, N \}$, where $N$ is the number of training samples and $\xi_i$ is the $i^{th}$ training sample.
\\
\textbf{Assumption 1.} \textit{(Bounded variance): There exists a constant $\sigma \geq 0$ that satisfies $\mathbb{E}[\| \nabla \ell_i(w) - \nabla L(w) \|^2]  \leq \sigma^2$ for all $i \in \{1, \cdots, N\}$ and $w \in \mathbb{R}^d$.}
\\
\textbf{Assumption 2.} \textit{($\beta$-smoothness): There exists a constant $\beta \geq 0$ that satisfies $\| \nabla \ell_i(u) - \nabla \ell_i(v) \| \leq \beta \| u - v \| $ for all $i \in \{1, \cdots, N\}$ and $w \in \mathbb{R}^d$.}


Then, we present the performance guarantee of mini-batch SGD when the proposed layer-wise gradient norm penalizing method is applied as follows. 
\\
\textbf{Theorem 1.} \textit{Assume the $\beta$-smooth loss function and the bounded gradient variance. Then, if $\eta \leq \frac{1}{2\beta}$ and $r \leq \frac{1}{2\beta}$, Algorithm \ref{alg:LASAM} with a batch size of $b$ satisfies:}
\begin{align}
    \frac{1}{T} \sum_{t=0}^{T-1} \mathbb{E} \left[ \| \nabla L(w_t) \|^2 \right] &\leq \frac{4}{T\eta} \left( L(w_0) - \mathbb{E}\left[ L(w_T) \right] \right) \nonumber \\
    &\hspace{1cm}+ 4\left(\eta \beta + \beta^2 r^2 \right) \frac{\sigma^2}{b}. \label{theorem:sigma}
\end{align}
See the Appendix for proof.
\\
\textbf{Remark 1.} (Finite Horizon Result) With diminishing $\eta$ and $r$ such that $\eta = \frac{1}{\beta \sqrt{T}}$ and $r = \frac{1}{\beta \sqrt[4]{T}}$ we guarantee the convergence of Algorithm \ref{alg:LASAM} as follows.
\begin{align}
    \frac{1}{T} \sum_{t=0}^{T-1} \mathbb{E} \left[ \| \nabla L(w_t) \|^2 \right] &\leq \frac{4\beta}{\sqrt{T}} \left( L(w_0) - \mathbb{E}\left[ L(w_T) \right] \right) \nonumber \\
    & \hspace{1cm} + \frac{8 \sigma^2}{b \sqrt{T}}. \nonumber
\end{align}
Thus, with a sufficiently small diminishing $\eta$ and $r$, the right-hand side has a complexity of $\mathcal{O}(\frac{1}{\sqrt{T}})$, which is the same as the typical convergence complexity of SAM \cite{andriushchenko2022towards,sun2023adasam,jiang2023adaptive}.
\\
\textbf{Remark 2.} (Layer-wise Model Perturbation) The result shows that the convergence rate is maintained regardless of which layers are perturbed.
Unless the gradient goes to zero at all the layers that are not perturbed, we can expect that the second term on the right-hand side in (\ref{theorem:sigma}) will be most likely smaller than the derived bound in practice, making it converge faster.
Thus, we can conclude that users have the freedom to choose which layers to perturb, focusing only on how to strike a good balance between generalization performance and the model perturbation cost.

Now, in the following subsection, let us focus on how to find the minimal set of layers $w'$ that achieves similar generalization performance to the full gradient norm penalizing method.

\begin{figure}[t]
\centering
\includegraphics[width=\columnwidth]{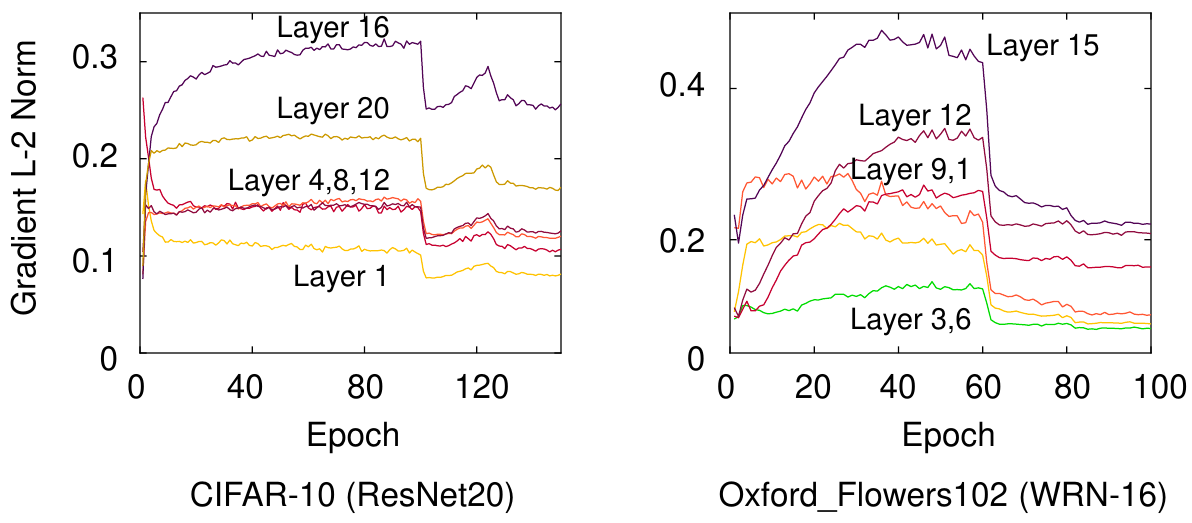}
\caption{
    The layer-wise gradient norm curves of a) CIFAR-10 (ResNet20) training and b) Oxford\_flowers102 (Wide-ResNet16).
    All the layers show consistent gradient norms throughout the training epochs.
}
\label{fig:layer_norm}
\end{figure}

\begin{figure*}[t]
\centering
\includegraphics[width=1.3\columnwidth]{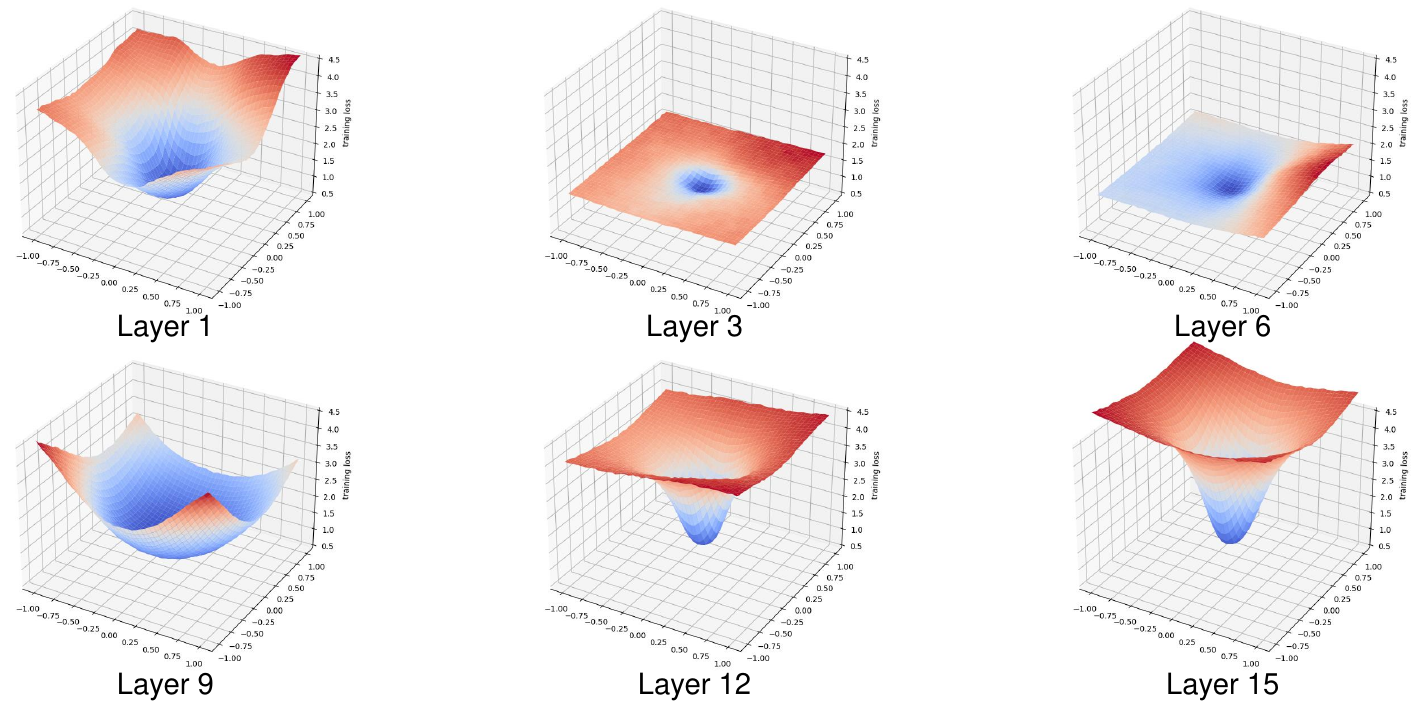}
\caption{
    The loss landscape of Wide-ResNet16 trained on Oxford\_Flowers102. Each layer has noticeably different loss landscapes.
}
\label{fig:landscape}
\end{figure*}

\subsection {Adaptive Layer Selection}
\textbf{Layer-Wise Data Characteristics Analysis} --
Our empirical study finds that the gradient norm delivers useful information regarding how individual layers contribute to the model's generalization performance.
Specifically, the layers have a consistent magnitude of the gradients throughout the training and the magnitude is quite stable while the learning rate remains the same.
Figure \ref{fig:layer_norm} shows the layer-wise gradient norm curves of CIFAR-10 \cite{krizhevsky2009learning} (ResNet-20 \cite{he2016deep}) and Oxford\_flowers102 \cite{Nilsback08} (WRN-16 \cite{zagoruyko2016wide}).
We see that some output-side layers have a noticeably larger gradient norm than others.
Assuming a strong correlation between the gradient norm and the model's generalization performance \cite{zhao2022penalizing,zhang2023gradient}, we can consider that those layers strongly affect the model's generalization performance.
Even if the gradient penalizing method is applied to some layers, if their gradient norms are small, they will have a non-negligible extra computational cost while not making a meaningful impact on the generalization performance.

Figure \ref{fig:landscape} shows the layer-wise loss landscapes measured from Wide-ResNet16 training on Oxford\_Flowers102.
We used the loss visualization method proposed in \cite{li2018visualizing}.
For each layer, 1) we prepare two orthogonal random vectors, 2) adjust each layer using them, and then 3) collect the training loss at $40 \times 40$ different grid points in the parameter space.
See the further details about the visualization settings.
Interestingly, we see that the layers show extremely different loss landscapes.
Such a tendency is well aligned with what we observed on the layer-wise gradient norm curves shown in Figure \ref{fig:layer_norm}.
These visualized loss landscapes provide insights into which layers should be perturbed to maximize the efficacy of the layer-wise gradient norm penalization.
\\
\textbf{Adaptive Layer Selection Algorithm} --
The above analysis leads us to design a layer-wise adaptive gradient penalizing method as follows.
After every model update, we first collect the gradient norm at all individual layers.
Then, choose $k$ layers, $w'$ in (\ref{eq:subset}), that have the largest gradient norm among $L$ layers.
In the following iteration, perturb only the layers in $w'$ instead of the total model.
Algorithm \ref{alg:LASAM} shows the described steps.
The algorithm selectively perturbs the layers using gradient ascent and then calculates the gradient descent at lines $5 \sim 7$.
The layer-wise gradient norm values are collected at line 9.
Note that the extra computation of collecting the gradient norm is almost negligible as compared to the forward pass or the backward pass cost.

\begin{algorithm}[t]
\begin{algorithmic}[1]
\caption{
    Layer-wise Adaptive Gradient Norm Penalizing Method (Layer-wise SAM)
}
\label{alg:LASAM}
\STATE \textbf{Input}: $k$: the number of layers to perturb.
\STATE {Initialize the layer-wise gradient norms to 0.}
    \FOR{$t = 0$ to $T-1$}
        \STATE {$w' \leftarrow k$ layers with the top gradient norms.}
        \STATE {Compute $\nabla L_S(w_i), \hspace{0.2cm} \forall i \in w'$}
        \STATE {Compute $\nabla L_S(w_i + r\nabla L_S(w_i)), \hspace{0.2cm} \forall i \in w$}
        \STATE {Compute the update: $\nabla L(w_i)$ using Eq. \ref{eq:grads3}} 
        \STATE {Update the model $w$ using $\nabla L(w_i)$}
        \STATE {Collect $\| \nabla L_S(w_i + r\nabla L_S(w_i)) \|_2, \hspace{0.2cm} \forall i \in w$}
    \ENDFOR
    \STATE{\textbf{Output:} $w$}
\end{algorithmic}
\end{algorithm}

We consider $k$ to be a tunable hyper-parameter.
If $k = L$, it becomes the conventional SAM.
On the other hand, if $k = 0$, it becomes the vanilla SGD.
As $k$ increases, the model perturbation cost increases while the gradient norm is more strongly penalized.
However, as shown in Figure \ref{fig:layer_norm}, some layers tend to have an almost negligible gradient magnitude.
Thus, users have to find the minimum value of $k$ that sufficiently suppresses the gradient norm having a minimal perturbation cost.
We found that the $k$ was tuned to a small value between $2 \sim 8$ across many different datasets, not only the computer vision benchmarks but also language modeling benchmarks.
Thus, one can begin the tuning with $k=4$ and easily find the best setting by a few grid searches.

As compared to the SOTA parameter-wise perturbation methods \cite{du2021efficient,mi2022make}, our layer-wise method has three critical benefits.
First, the layer-wise method clearly reduces the gradient computation cost without any help of sparse matrix multiplication features.
We believe that such a low implementation complexity will encourage SAM techniques to be more easily adopted to real-world applications.
Second, the error propagation cost can be significantly reduced.
The parameter-wise method still should propagate the errors through all the layers, limiting the performance gain.
Finally, our gradient norm-based layer selection enables users to make a practical trade-off between the model perturbation cost and the generalization performance.
By perturbing the layers with a high gradient norm, we expect the gradient norm of the whole model to be effectively suppressed while considerably reducing the total model perturbation cost.

%% file: 4_eval.tex
\begin{table*}[ht!]
\centering
\small
\begin{tabular}{llccccccc} \toprule
Dataset & Model & Batch size & L.R. & Epochs & Vaniila SGD & Generalized SAM & LookSAM & \textbf{Layer-wise SAM} \\ \midrule 
CIFAR-10 & ResNet20 & 128 & 0.1 & 150 & $91.96 \pm 0.3\%$ & $92.64 \pm 0.4\%$ & $92.14 \pm 0.4\%$ (3) & \textbf{92.81}$\pm 0.3\%$ (4) \\
CIFAR-100 & Wide-ResNet28 & 128 & 0.1 & 200 & $79.41 \pm 0.4\%$ & $80.83 \pm 0.5\%$ & $80.13 \pm 0.3\%$ (2) & \textbf{81.25}$\pm 0.4\%$ (4) \\
Oxford\_Flowers102 & Wide-ResNet16 & 40 & 0.1 & 100 & $69.82 \pm 0.4\%$ & $74.77 \pm 0.5\%$ & $73.53 \pm 0.6\%$ (3) & \textbf{75.56}$\pm 0.6\%$ (2) \\ 
CIFAR-100 (Fine-Tuning) & ViT-b16 & 128 & 0.001 & 10 & $90.73 \pm 0.5\%$ & $91.46 \pm 0.9\%$ & $91.15 \pm 0.7\%$ (2) & \textbf{91.36}$\pm 0.5\%$ (1) \\ \bottomrule
\end{tabular}
\caption{
    The image classification performance comparison.
}
\label{tab:compare}
\end{table*}

\begin{figure*}[t]
\centering
\includegraphics[width=1.9\columnwidth]{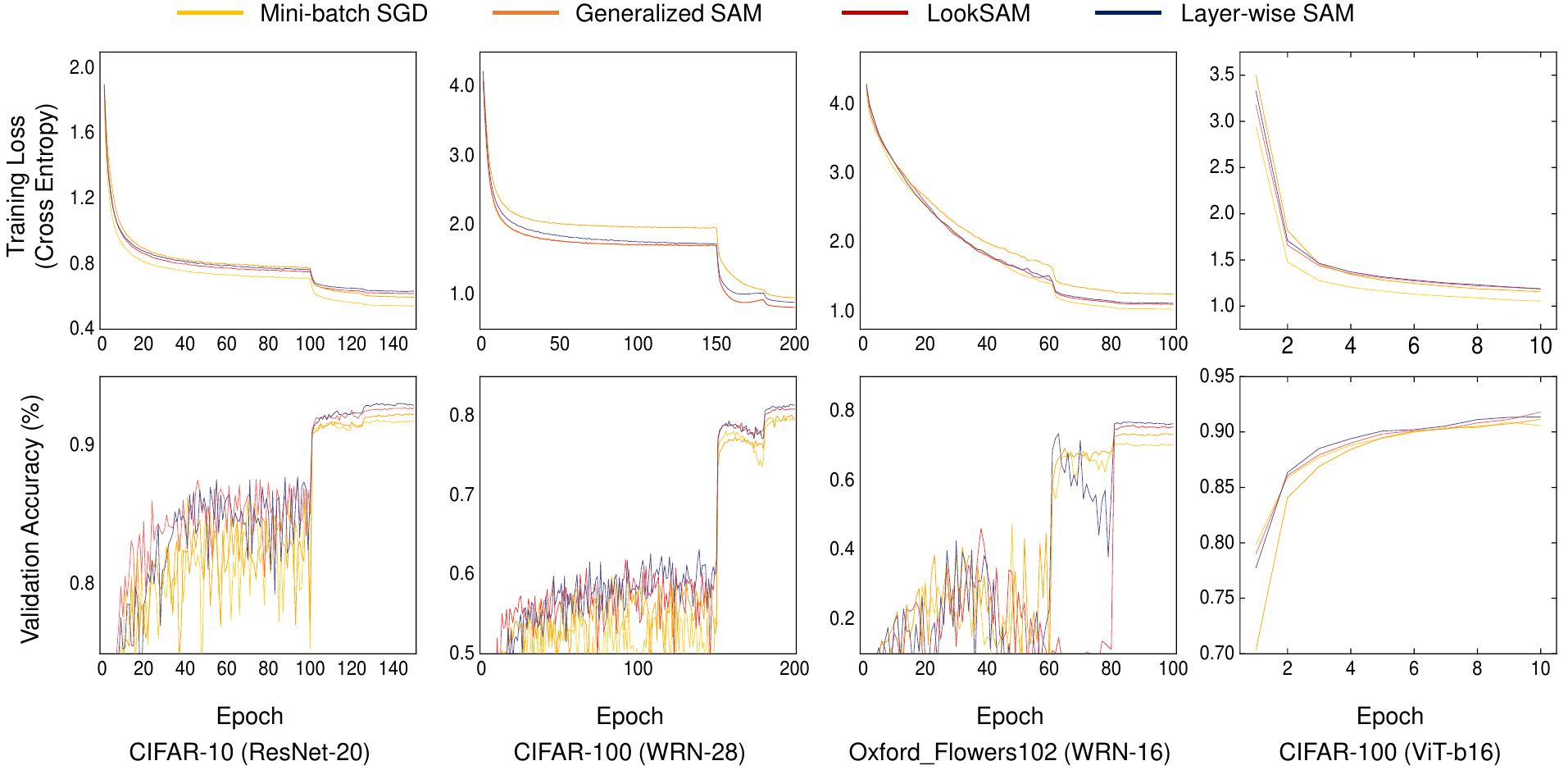}
\caption{
    The learning curves of CIFAR-10, CIFAR-100, and Oxford\_Flowers102. The hyper-parameters are shown in Table \ref{tab:compare}.
}
\label{fig:learning}
\end{figure*}

\section {Performance Evaluation} \label{sec:eval}
\subsection {Experimental Settings}
\textbf{System Configuration} -- We evaluate the performance of the proposed method using TensorFlow 2.9.3.
The training software is written in Python and parallelized with MPI to use two NVIDIA RTX A6000 GPUs.
The training throughput (images / sec) takes into account both the computation time for local training and the MPI communication time for averaging the locally calculated gradients across the GPUs.
The reported accuracy results are averaged across at least three independent runs.
\\
\textbf{Dataset and Hyper-Parameters} -- We use three representative computer vision benchmarks, CIFAR-10 \cite{krizhevsky2009learning}, CIFAR-100, and Oxford\_Flowers102 \cite{Nilsback08}.
For the three datasets, we use ResNet-20 \cite{he2016deep}, Wide-ResNet28-10 (WRN28) \cite{zagoruyko2016wide}, and Wide-ResNet16-2 (WRN16), respectively.
In addition, we also run fine-tuning experiments using Vision Transformer (ViT)~\cite{dosovitskiy2020image} and CIFAR-100.
We attached two fully-connected layers at the end of the ViT-b16 pretrained on ImageNet and then fine-tune only the new layers.
The batch size is 128 for both CIFAR datasets and 40 for Oxford\_Flowers102.
The learning rate starts at 0.1 and is decayed by a factor of 10 twice.
The number of epochs is 150, 200, and 100, for the three benchmarks respectively.
Because the TensorFlow version of Oxford\_Flowers102 contains only $12\%$ of the total dataset as training data, we instead used the test data that takes up $\sim 70\%$ for training.
The validation accuracy is measured using the rest of the dataset.

\subsection {Comparative Study}
We compare the proposed method to the SOTA variants of SAM, generalized SAM \cite{zhao2022penalizing} and LookSAM \cite{liu2022towards}, as well as the conventional mini-batch SGD.
Table \ref{tab:compare} shows the performance comparison and the hyper-parameter settings and Figure \ref{fig:learning} shows the corresponding learning curves.
The generalized SAM has two hyper-parameters, $r$ and $\alpha$, as shown in (\ref{eq:grads2}).
We tuned them by a grid search with a unit of 0.1, and the best setting was $r=0.1$ and $\alpha=0.8$ for CIFAR datasets and $r=0.1$ and $\alpha=0.7$ for Oxford\_Flowers102.
The numbers in the bracket on LookSAM column indicate the gradient ascent re-calculation interval.
We set the re-calculation interval to the maximum value that gives higher accuracy than the mini-batch SGD, expecting the maximized throughput without much losing the accuracy.
The numbers in the bracket on Layer-wise SAM column indicate the number of layers to which the gradient norm penalizing method is applied.

Overall, the proposed method, \lq{}Layer-wise SAM\rq{} achieves validation accuracy comparable to or even slightly higher than the generalized SAM in all the four benchmarks.
This result empirically proves that a few critical layers strongly affect the model's generalization performance.
By applying the gradient norm penalizing method to only a few critical layers, our method efficiently suppresses the gradient norm, achieving superior validation accuracy.
Regardless of how many layers employ the gradient norm penalizing method, the training loss converges more slowly than the mini-batch SGD.
This result is aligned with our understanding of convergence properties as shown in Section \ref{sec:theory}.
We found that LookSAM rapidly lost accuracy as the gradient ascent interval increased.
It achieved higher accuracy than the mini-batch SGD only when the interval was smaller than or equal to $3$ in all three benchmarks.

\begin{figure*}[ht!]
\centering
\includegraphics[width=1.7\columnwidth]{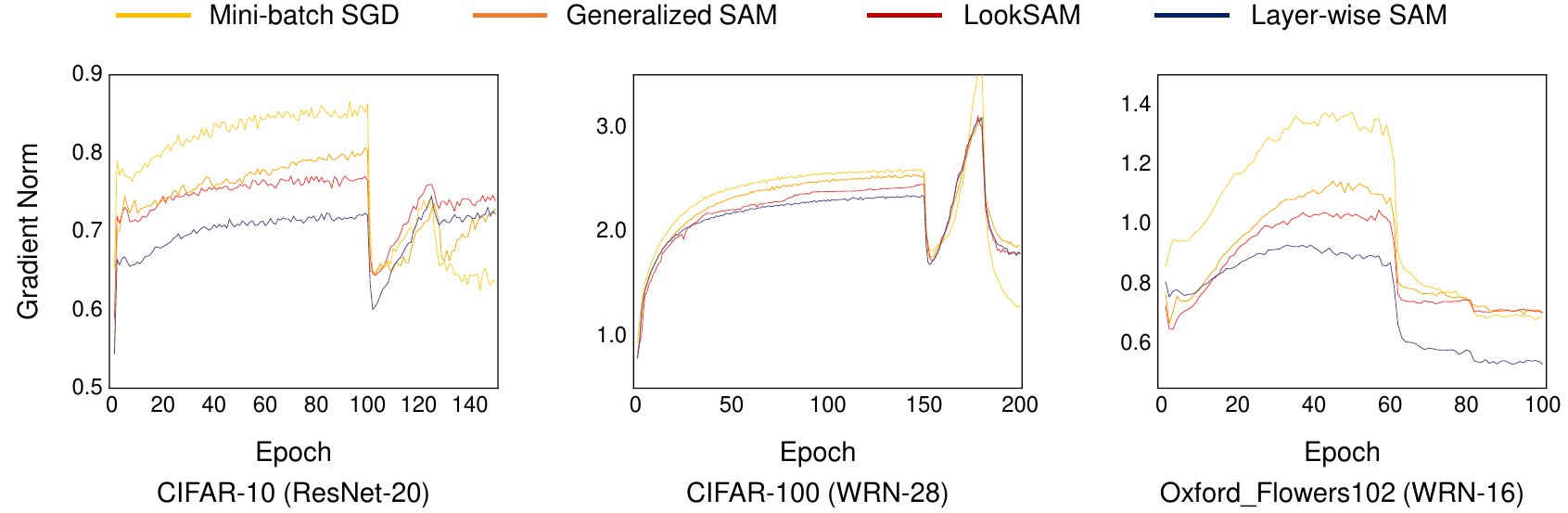}
\caption{
    The full model gradient norm curves of CIFAR-10 (ResNet20), CIFAR-100 (Wide-ResNet28), and Oxford\_Flowers102 (Wide-ResNet16).
    The norm is noticeably reduced when any SAM method is applied. The proposed layer-wise method effectively suppresses the gradient norm likely to the full model perturbation method.
}
\label{fig:norm}
\end{figure*}

\begin{figure*}[t]
\centering
\includegraphics[width=1.8\columnwidth]{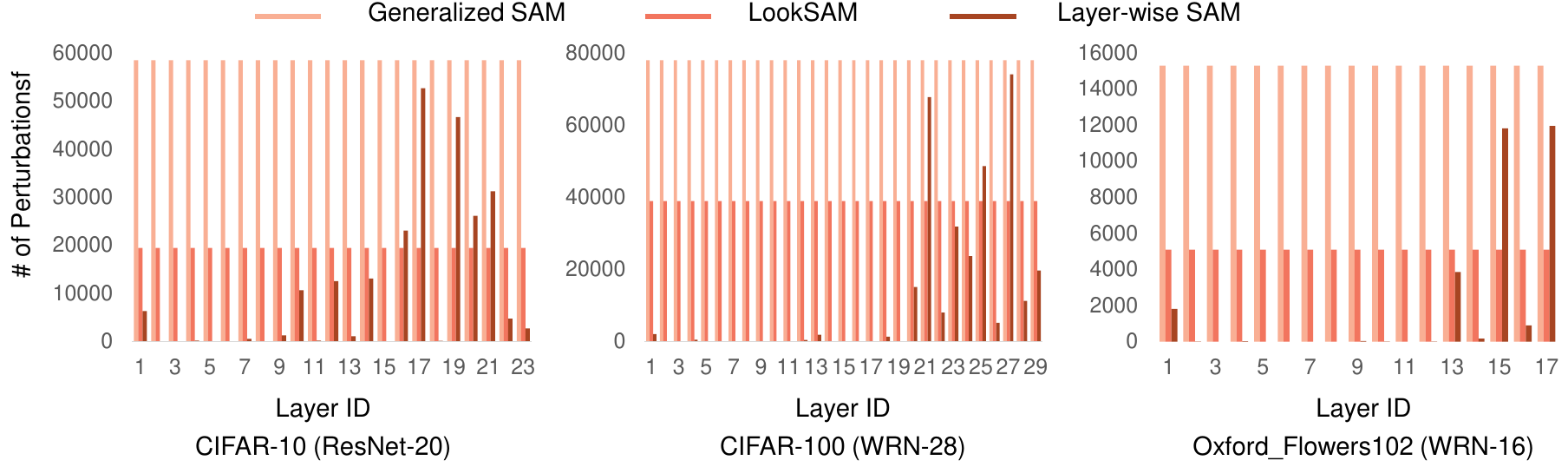}
\caption{
    The number of perturbations at all individual layers.
    The generalized SAM perturbs all the layers at every iteration. LookSAM's gradient ascent re-calculation interval is $3$, $2$, and $3$ for the three benchmarks, respectively. Our proposed method selectively perturbs the $k$ critical layers only, and the $k$ is set to $4$, $4$, and $2$ for the three benchmarks, respectively.
}
\label{fig:comp}
\end{figure*}

To further support our proposed method, we present the full model gradient norm comparison in Figure \ref{fig:norm}.
The proposed method effectively suppresses the gradient norm even though the model is perturbed at only a few layers.
We omit the ViT fine-tuning experiment here because the majority of the model parameters are frozen during the training.
One interesting observation is that the proposed method shows the gradient norm even lower than that of the generalized SAM that perturbs all the layers.
The same results are observed in all the three benchmarks.
This result implies that the model's generalization performance can be harmed if the layers with a minor gradient norm are perturbed.
A similar observation was reported in \cite{mi2022make}.

Another intriguing observation is that, while the gradient norm is well suppressed by applying the penalizing method to either the full model or a few critical layers, the norm becomes higher than that of the mini-batch SGD after decaying the learning rate.
We find that such a fast drop in the gradient norm is not aligned with how the generalization performance evolves.
Instead, the gradient norm collected before the learning rate decay effectively represents the generalization performance.
Understanding this phenomenon will be an interesting future work.

It is worth noting that our method consistently achieves slightly higher accuracy as compared to the generalized SAM in all the benchmarks.
Similar results have been already reported in \cite{du2021efficient}.
This common tendency implies that perturbing certain parameters may rather hurt the generalization performance.
Our experimental results show that the layers with a small gradient norm can be considered as such sensitive layers that may hurt the accuracy when perturbed.
Theoretically explaining this observation is another critical future work.

\begin{table*}[t]
\centering
\small
\begin{tabular}{llccccc} \toprule
Dataset & Model & Mini-batch SGD & Generalized SAM & LookSAM & Layer-wise SAM \\ \midrule
CIFAR-10 & ResNet20 & 659.2 img/sec & 459.5 img/sec & 595.2 img/sec & \textbf{604.2} img/sec \\
CIFAR-100 & Wide-ResNet28 & 229.1 img/sec & 139.5 img/sec & 151.0 img/sec & \textbf{190.7} img/sec \\
Oxford\_Flowers102 & Wide-ResNet16 & 604.2 img/sec & 401.9 img/sec & \textbf{473.6} img/sec & 463.4 img/sec \\ 
CIFAR-100 (Fine-Tuning) & ViT-b16 & 77.28 img/sec & 49.82 img/sec & 50.32 img/sec & \textbf{53.97} img/sec \\
\bottomrule
\end{tabular}
\caption{
    The training throughput on two NVIDIA RTX A6000 GPUs. One process is pinned on a single GPU and they synchronize the locally calculated gradients using MPI \textit{allreduce} operations. Since the communication is performed on a single compute node, the communication time is almost negligible.
}
\label{tab:throughput}
\end{table*}

\subsection {Computational Cost Analysis}
\textbf{Model Perturbation Cost} --
Let us take a closer look at how the proposed method reduces the computational cost of the model perturbation.
Figure \ref{fig:comp} shows the number of model perturbations at all individual layers.
Note that the proposed method, \lq{}Layer-wise SAM\rq{} achieves accuracy similar to or even slightly higher than that of the \lq{}SAM\rq{} the full model perturbation method.
The \lq{}Layer-wise SAM\rq{} has remarkably fewer model perturbations at most of the layers compared to the other methods.
We also see that the output-side layers are more likely chosen than the input-side layers regardless of the dataset.
This tendency means that the output-side layers likely have a larger gradient norm than the input-side layers.
In the backward pass, the errors are back-propagated and then the gradients can be computed using the obtained errors and the activations received from the previous layer.
Thus, our method can reduce the error propagation cost together with the gradient computation cost, and it results in significantly reducing the total model perturbation cost.
\\
\textbf{Theoretical Analysis and Comparisons} --
The computational cost of the proposed layer-wise perturbation method is as follows.
\begin{align}
    \mathcal{B}_{\text{min}(w')} + \sum_{i\in w'} | \nabla L(w_i) | + \mathcal{C}, \label{eq:cost1}
\end{align}
where $\mathcal{B}_{i}$ is the error propagation cost from the output layer to the layer $i$.
The $|\cdot|$ indicates the number of elements and the $\mathcal{C}$ indicates a constant computational cost that a single backward pass deterministically has.
The left term in (\ref{eq:cost1}) indicates the error backpropagation cost to the most input-side layer among $w'$.
So, the $\mathcal{B}_{\text{min}(w')}$ heavily depends on which layers are selected to be perturbed.
The right term is the gradient computation cost at all the $k$ layers in $w'$.

\begin{figure*}[t]
\centering
\includegraphics[width=1.3\columnwidth]{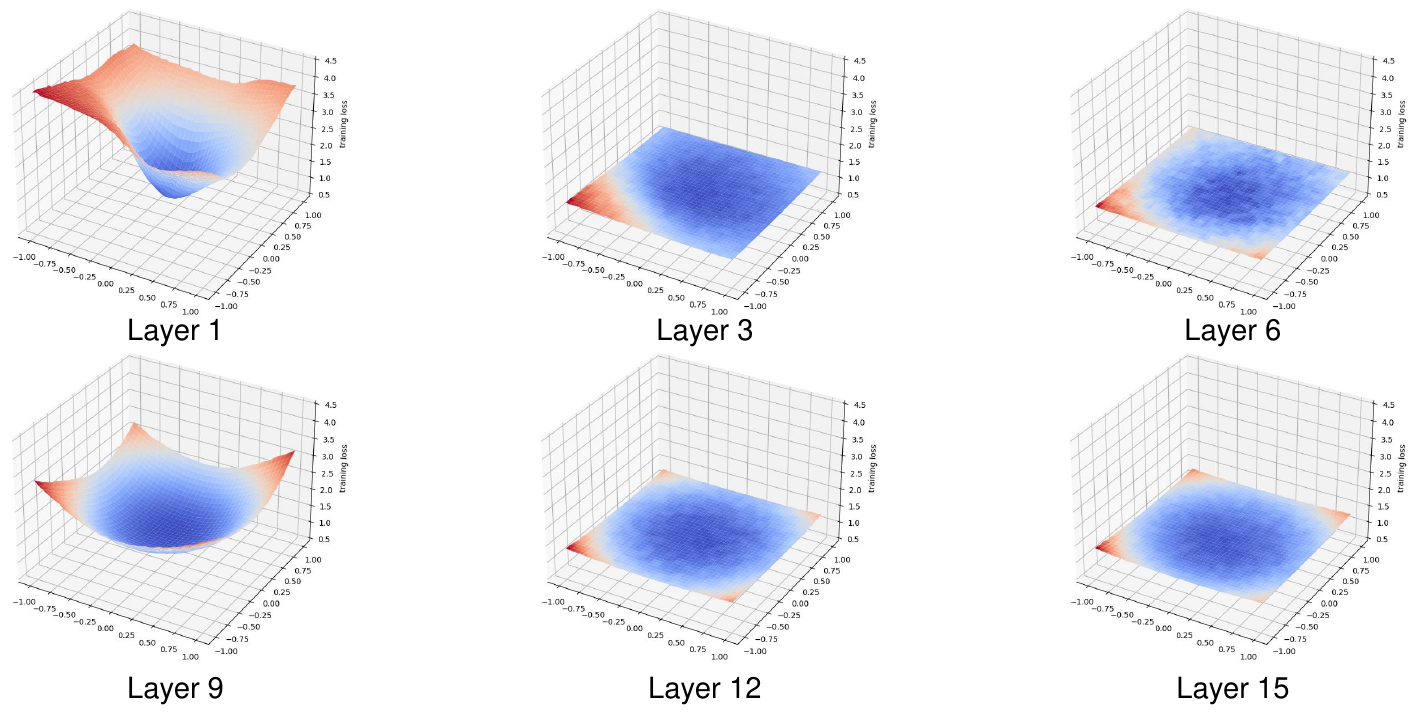}
\caption{
    The loss landscape of Wide-ResNet16 trained on Oxford\_Flowers102. The model was trained using the proposed layer-wise gradient norm penalizing method. The landscapes are much flatter as compared to Figure \ref{fig:landscape}.
}
\label{fig:landscape2}
\end{figure*}

To the best of our knowledge, \cite{liu2022towards} shows the best computational efficiency among several SAM variants.
It has a hyper-parameter $k$ which determines how many times the gradients will be re-used to perturb the model parameters.
The authors suggest $k=5$ which generally works well across several benchmarks.
Their experimental results show that the accuracy drops significantly when $k > 5$ in all the benchmarks.
Their computational cost can be analyzed as follows.
\begin{align}
\frac{1}{k} \left( \mathcal{B}_{1} + \sum_{i \in w} | \nabla L(w_i) | \right) + \mathcal{C}. \label{eq:cost2}
\end{align}
The $\mathcal{B}_1$ indicates the error propagation cost from the output layer to the input layer.
So, the total perturbation cost is proportionally reduced as $k$ increases.
The costs of these two different approaches, (\ref{eq:cost1}) and (\ref{eq:cost2}), cannot be easily compared because the value of (\ref{eq:cost1}) is dependent on the model architecture and which layers are included in $w'$.
Furthermore, the actual throughput shown in \cite{liu2022towards} is also not proportional to $k$ because many constant factors strongly affect the training time.
For example, ViT \cite{dosovitskiy2020image} throughput of ImageNet-1K \cite{deng2009imagenet} training is improved roughly from $13,000$ to $19,000$ as the $k$ goes from $1$ to $5$ ($31.6\%$ improved).
Therefore, we focus on the actual throughput rather than the theoretical computational complexity.

ESAM proposed in \cite{du2021efficient} perturbs all individual parameters independently with a probability of $\beta$.
However, due to the computations that have a constant complexity, its performance gain is not supposed to be similar to the ideal speedup.
For example, when $\beta = 0.5$, their reported ViT-S/16 (ImageNet) training throughput of ESAM is 734 images/sec while that of the conventional SAM is 581 images/sec ($26.3\%$ improved).
In addition, this approach assumes that the underlying deep learning framework supports the sparse gradient computation feature.
Thus, we do not directly compare the performance of our method to ESAM.
\\
\textbf{Training Throughput} --
Table \ref{tab:throughput} shows the training throughput measured on two A6000 GPUs.
The corresponding accuracy is shown in Table \ref{tab:compare}.
The $k$ of the proposed method is set to $4$, $4$, $2$, and $1$ for the four benchmarks, respectively.
The gradient ascent interval of LookSAM is set to $3$, $4$, $3$, and $2$ for the same four benchmarks, respectively.
Overall, our method achieves similar accuracy to the full model perturbation method while significantly improving the training throughput.
LookSAM achieves a slightly higher throughput of Oxford\_Flowers102 but its accuracy is much lower than the proposed method as shown in Table \ref{tab:compare}.
For the other three benchmarks, our method achieves the best throughput together with the highest accuracy.

Interestingly, the proposed method shows a quite different performance improvement depending on the model architecture.
The throughput drop for for ResNet20, WRN28, and WRN16, are $8.5\%$, $16.8\%$, and $25.4\%$, respectively.
The networks have their special architectures and the training datasets also have different patterns.
If a certain combination of dataset and model causes a large gradient norm at the input-side layers, it will likely have an expensive model perturbation cost due to the long backward passes.
On the other hand, if the combination yields a large gradient norm at the output-side layers, the backward pass for the layer-wise model perturbation becomes short having a cheaper computational cost.
In our experiments, all the benchmarks showed a different distribution of gradient norms across the layers, and it resulted in having such a different model perturbation cost.

\subsection {Ablation Study}
\textbf{Loss Landscape} --
Figure \ref{fig:landscape2} shows the layer-wise loss landscape when the model is trained using the proposed layer-wise gradient norm penalizing method.
The training loss values were collected from the same grid points as Figure \ref{fig:landscape}.
The plots show that the proposed method leads the model, especially the layers that fell into a sharp valley when using mini-batch SGD, to a flat region in the parameter space.
This result is also well aligned with what we observed in Figure \ref{fig:comp} such that the corresponding output-side layers are mostly chosen to be perturbed and their loss landscapes become much flatter as compared to Figure \ref{fig:landscape}.
\\
\textbf{Layer Selection Criterion} --
To verify the efficacy of the proposed adaptive layer selection method, we compare the performance of three different layer-wise perturbation approaches: 1) perturb the layers with a high gradient norm, 2) perturb the layers with a low gradient norm, and 3) random selection.
Table \ref{tab:ablation} shows the accuracy comparisons.
Note that \textit{None} is the mini-batch SGD.
The \textit{Random} can be considered to be a coarse-grained version of the randomized parameter-wise perturbation in \cite{du2021efficient}.
The $k$ is set to the smallest value that gives a similar accuracy to the conventional SAM.

In all three benchmarks, the \textit{Low-Norm} achieves significantly lower accuracy than the other perturbation criteria.
This result empirically proves that putting more weight on the layers with a high gradient norm is a valid approach to improve the generalization performance.
We also see that \textit{Random} shows a non-negligible accuracy improvement as compared to the mini-batch SGD.
By randomly choosing the layers to perturb, some critical layers could be included by chance, and it results in improving the generalization performance.
Since the \textit{High-Norm} consistently achieves higher accuracy than \textit{Random} in all the benchmarks, we can conclude that the proposed method effectively finds the best trade-off between the model perturbation cost and the generalization performance.
\\
\textbf{Impact of the Number of Layers to Perturb, $k$} --
We adjust the $k$ setting and analyze its impact on the generalization performance.
As the gradient norm penalizing method is applied to more layers, the norm is expected to be suppressed more strongly.
Table \ref{tab:k} shows the accuracy comparisons.
Interestingly, the accuracy becomes similar to the conventional SAM when $k$ is even set to a very small value such as 1 or 2.
Once it achieves the best accuracy at a certain $k$, the accuracy does not further go up as $k$ increases.
However, the throughput monotonically decreases as $k$ increases.
Therefore, users should find the value of $k$ that gives the best accuracy at the minimal model perturbation cost
In practice, we suggest starting the tuning from $k=2$ and then conducting a simple grid search to find a better setting.

\begin{table}[t]
\centering
\small
\begin{tabular}{lccc} \toprule
Dataset & $k$ & Perturbation Criterion & Validation Acc. (\%) \\ \midrule
\multirow{4}{*}{CIFAR-10} & \multirow{4}{*}{4} & None & $91.63 \pm 0.3\%$ \\
& & High-Norm & $\textbf{92.81} \pm 0.3\%$  \\
& & Low-Norm & $91.40 \pm 0.5\%$ \\
& & Random & $92.46 \pm 0.2\%$ \\ \midrule
\multirow{4}{*}{CIFAR-100} & \multirow{4}{*}{4} & None & $79.41 \pm 0.4\%$ \\
& & High-Norm & $\textbf{81.25} \pm 0.4\%$ \\
& & Low-Norm & $80.01 \pm 0.9\%$ \\
& & Random & $80.67 \pm 0.5\%$ \\ \midrule
\multirow{4}{*}{Oxford\_Flowers102} & \multirow{4}{*}{2} & None & $69.82 \pm 0.4\%$ \\
& & High-Norm & $\textbf{75.36} \pm 0.6\%$ \\
& & Low-Norm & $71.76 \pm 1.1\%$ \\
& & Random & $72.45 \pm 0.9\%$ \\ \bottomrule
\end{tabular}
\caption{
    The validation accuracy comparisons among a) \textit{None}, b) \textit{High-Norm}, c) \textit{Low-Norm}, and d) \textit{Random} layer-wise perturbation methods. The best accuracy is achieved when the layers with the highest gradient norm are perturbed in all three benchmarks.
}
\label{tab:ablation}
\end{table}

\begin{table}[t]
\centering
\small
\begin{tabular}{lccc} \toprule
Dataset & $k$ & Validation Acc. (\%) & Throughput (img/sec) \\ \midrule
\multirow{5}{*}{CIFAR-10} & 1 & $92.51 \pm 0.5\%$ & 628.5 \\
& 2 & $92.57 \pm 0.4\%$ & 617.0 \\
& 4 & $92.81 \pm 0.3\%$ & 604.2 \\
& 8 & $92.77 \pm 0.5\%$ & 564.5 \\
& all & $92.64 \pm 0.4\%$ & 459.5 \\ \midrule
\multirow{5}{*}{CIFAR-100} & 1 & $80.12 \pm 0.3\%$ & 229.1 \\
& 2 & $80.81 \pm 0.7\%$ & 194.6 \\
& 4 & $81.25 \pm 0.4\%$ & 190.7 \\
& 8 & $81.93 \pm 0.5\%$ & 160.0 \\
& all & $80.83 \pm 0.5\%$ & 139.5 \\ \midrule
\multirow{5}{*}{Oxford\_Flowers102} & 1 & $74.93\pm 0.5\%$ & 482.6 \\
& 2 & $75.56 \pm 0.6\%$ & 463.4 \\
& 4 & $76.31 \pm 1.1\%$ & 442.9 \\
& 8 & $76.70 \pm 0.9\%$ & 427.5 \\ 
& all & $74.77 \pm 0.5\%$ & 401.9 \\ \bottomrule
\end{tabular}
\caption{
    The validation accuracy and throughput with different $k$ settings.
    In all the benchmarks, the best accuracy is achieved with a certain $k$ and then it does not further go up as $k$ increases.
}
\label{tab:k}
\end{table}

%% file: 5_conclusion.tex
\section {Conclusion} \label{sec:conclusion}

In this paper, we discussed how to tackle the expensive model perturbation cost issue in SAM.
By selectively perturbing the critical layers only, our method suppresses the gradient norm of the whole model while effectively improving the models' generalization performance.
This method is not dependent on model architecture, data characteristics, or optimization algorithms.
Thus, it can be employed by real-world deep learning applications without having any conflicts to their existing features.
Our empirical study also provides intriguing insights into how each layer contributes to the whole network's generalization performance.
We believe that our study will encourage many Deep Learning applications to take advantage of the SAM to enjoy improved generalization performance.

%% file: appendix.tex
\onecolumn
\appendix
\section{Appendix} \label{app:proof}

\subsection{Proof of Theorem}
We first present a couple of useful lemmas here.
Note that our analysis borrows the proof structure used in \cite{andriushchenko2022towards}.
\begin{lemma} \label{lemma:1}
Given a $\beta$-smooth loss function $L(x)$, we have the following bound for any $x \in \mathbb{R}^d$.
\begin{align}
    \langle \nabla L(u) - \nabla L(v), u - v \rangle & \geq -\beta \| u - v \|^2. \nonumber
\end{align}
\end{lemma}
\begin{proof}
Starting from the smoothness assumption,
\begin{align}
     \| \nabla L(u) - \nabla L(v) \| \leq \beta \| u - v \| \text{ for all } u \text{ and } v \in \mathbb{R}^d \nonumber
\end{align}
By multiplying $\| v - u \|$ on the both side, we get
\begin{align}
    \| \nabla L(u) - \nabla L(v) \| \| v - u \| & \leq \beta \| u - v \| \| v - u \| \nonumber \\
    \| \nabla L(u) - \nabla L(v) \| \| v - u \| & \leq \beta \| u - v \|^2 \nonumber \\
    \langle \nabla L(u) - \nabla L(v), v - u \rangle & \leq \beta \| u - v \|^2, \label{eq:CS1} \\
    \langle \nabla L(u) - \nabla L(v), u - v \rangle & \geq -\beta \| u - v \|^2. \nonumber
\end{align}
where (\ref{eq:CS1}) is based on Cauchy-Schwarz inequality.
\end{proof}
Here, we additionally define the magnitude of the gradients of $w'$ as follows.
\begin{align}
    \| \nabla L( w' ) \|^2 = \sum_{i \in w'} \| \nabla L(w_i) \|^2 = \kappa \| \nabla L(w) \|^2, 0 \leq \kappa \leq 1 \nonumber
\end{align}
The $\kappa$ is defined based on the model architecture and which layers are selected to be perturbed.
That is, as $k$ increases, $\kappa$ will also increase according to the number of parameters at the chosen layers.
Without loss in generality, we define $\kappa$ as a ratio of the number of perturbed parameters to that of the total parameters.
In this way, our analysis can cover any possible model architecture and the input data.

\begin{lemma} \label{lemma:2}
Given a $\beta$-smooth loss function $L(x)$, we have the following bound for any $r > 0$ and $x \in \mathbb{R}^d$.
\begin{align}
    \langle \nabla L(w + r \nabla L(w')), \nabla L(w) \rangle \nonumber \geq (1 - r\beta\kappa) \| \nabla L(w) \|^2 \nonumber
\end{align}
\end{lemma}
\begin{proof}
\begin{align}
    \langle \nabla L(w + r \nabla L(w')), \nabla L(w) \rangle \nonumber &= \langle \nabla L(w + r \nabla L(w')) - \nabla L(w), \nabla L(w) \rangle + \| \nabla L(w) \|^2 \nonumber \\
    &= \frac{1}{r} \langle \nabla L(w + r \nabla L(w')) - \nabla L(w), r \nabla L(w) \rangle + \| \nabla L(w) \|^2 \nonumber \\
    & \geq - \frac{\beta}{r} \| r \nabla L(w') \|^2 + \| \nabla L(w) \|^2 \label{eq:uselm1} \\
    & \geq - r\beta \| \nabla L(w') \|^2 + \| \nabla L(w) \|^2 \nonumber \\
    & = - r\beta\kappa \| \nabla L(w) \|^2 + \| \nabla L(w) \|^2 \label{eq:k} \\
    & \geq (1 - r\beta \kappa) \| \nabla L(w) \|^2, \nonumber
\end{align}
where (\ref{eq:uselm1}) is based on Lemma \ref{lemma:1}.
The (\ref{eq:k}) holds by the definition of $\kappa$.
\end{proof}

\begin{lemma}
\label{lemma:dot}
We consider the classical SAM which uses the same mini-batch when calculating the gradient ascent and the gradient descent.
Then, given a $\beta$-smooth loss function $L(x)$, we have the following bound for any $r > 0$, any $0 \leq \kappa \leq 1$, and $x \in \mathbb{R}^d$.
\begin{align}
    \mathbb{E} \left[ \langle \nabla L_{t+1}(w + r \nabla L_{t+1}(w')), \nabla L(w) \rangle \right] \geq \left(\frac{1}{2} - r\beta \kappa \right) \| \nabla L(w) \|^2 - \frac{\beta^2 r^2 \sigma^2}{2b} \nonumber
\end{align}
\end{lemma}
\begin{proof}
We first define the layer-wise gradient ascent step $\Tilde{w} = w + r\nabla L(w')$, where $\nabla L(w')$ indicates the global gradients at a subset of network layers $H$.
\begin{align}
    \mathbb{E} \left[ \langle \nabla L_{t+1}(w + r \nabla L_{t+1}(w')), \nabla L(w) \rangle \right] &= \mathbb{E} \left[ \langle \nabla L(w + r \nabla L_{t+1}(w')) , \nabla L(w) \rangle \right] \nonumber \\
    &= \mathbb{E} \left[ \langle \nabla L(w + r \nabla L_{t+1}(w')) - \nabla L(\Tilde{w}) + \nabla L(\Tilde{w}), \nabla L(w) \rangle \right] \nonumber \\
    &= \underset{E_1}{\underbrace{ \mathbb{E} \left[ \langle \nabla L(w + r \nabla L_{t+1}(w')) - \nabla L(\Tilde{w}), \nabla L(w) \rangle \right] }} + \underset{E_2}{\underbrace{ \mathbb{E} \left[ \langle \nabla L(\Tilde{w}), \nabla L(w) \rangle \right] }}. \nonumber
\end{align}
Then, we will bound $E_1$ and $E_2$ separately.
First, $E_1$ is lower-bounded as follows.
\begin{align}
    E_1 &= \mathbb{E} \left[ \langle \nabla L(w + r \nabla L_{t+1}(w')) - \nabla L(\Tilde{w}), \nabla L(w) \rangle \right] \nonumber \\
    &\geq -\frac{1}{2} \mathbb{E}\left[ \| \nabla L(w + r \nabla L_{t+1}(w')) - \nabla L(\Tilde{w}) \|^2 \right] - \frac{1}{2} \mathbb{E} \left[ \| \nabla L(w) \|^2 \right] \nonumber \\
    &\geq -\frac{\beta^2}{2} \mathbb{E}\left[ \| w + r \nabla L_{t+1}(w') - \Tilde{w} \|^2 \right] - \frac{1}{2} \mathbb{E} \left[ \| \nabla L(w) \|^2 \right] \label{eq:beta} \\
    &= -\frac{\beta^2}{2} \mathbb{E}\left[ \| r \nabla L_{t+1}(w') - r \nabla L(w') \|^2 \right] - \frac{1}{2} \mathbb{E} \left[ \| \nabla L(w) \|^2 \right] \nonumber \\
    &\geq -\frac{\beta^2 r^2 \sigma^2}{2b} - \frac{1}{2} \mathbb{E} \left[ \| \nabla L(w) \|^2 \right], \label{eq:sigma}
\end{align}
where (\ref{eq:beta}) is based on the smoothness assumption. The final equality, (\ref{eq:sigma}), is based on the bounded variance assumption.
Then, $E_2$ is lower-bounded directly based on Lemma \ref{lemma:2} as follows.
\begin{align}
    E_2 &= \mathbb{E} \left[ \langle \nabla L(\Tilde{w}), \nabla L(w) \rangle \right] \geq (1 - r\beta \kappa) \| \nabla L(w) \|^2. \nonumber
\end{align}
Summing up $E_1$ and $E_2$ bounds, we have
\begin{align}
    \mathbb{E} \left[ \langle \nabla L_{t+1}(w + r \nabla L_{t+1}(w')), \nabla L(w) \rangle \right] &\geq -\frac{\beta^2 r^2 \sigma^2}{2b} - \frac{1}{2} \mathbb{E} \left[ \| \nabla L(w) \|^2 \right] + (1 - r\beta \kappa) \| \nabla L(w) \|^2 \nonumber \\
    & = \left(\frac{1}{2} - r\beta\kappa \right) \| \nabla L(w) \|^2 - \frac{\beta^2 r^2 \sigma^2}{2b} \nonumber
\end{align}
\end{proof}

\begin{lemma}
\label{lemma:frame}
Under the assumption of $\beta$ smoothness and the bounded variance, the SAM guarantees the following if $\eta \leq \frac{1}{2\beta}$ and $r \leq \frac{1}{2\beta}$.
\begin{align}
    \mathbb{E}\left[ L(w_{t+1}) \right] \leq \mathbb{E}\left[ L(w_t) \right] - \frac{\eta}{4} \mathbb{E} \left[ \| \nabla L(w_t) \|^2 \right] + \eta \beta (\eta + \beta r^2) \frac{\sigma^2}{b} \label{eq:lemma4}.
\end{align}
\begin{proof}
Let us first define the model updated with the gradient ascent as $w_{t+1/2} = w_t + r\nabla L_{t+1}(w_t)$.
From the smoothness assumption, we begin with the following condition.
\begin{align}
    L(w_{t+1}) \leq L(w_t) - \eta \langle \nabla L_{t+1}(w_{t+1/2}), \nabla L(w_t) \rangle + \frac{\eta^2 \beta}{2} \| \nabla L_{t+1}(w_{t+1/2}) \|^2. \nonumber
\end{align}
Taking the expectation on both sides and based on the bounded variance assumption,
\begin{align}
    \mathbb{E}\left[ L(w_{t+1}) \right] & \leq \mathbb{E} \left[ L(w_t) \right] - \eta \mathbb{E} \left[ \langle \nabla L(w_{t+1/2}), \nabla L(w_t) \rangle \right] + \frac{\eta^2 \beta}{2}\mathbb{E} \left[ \| \nabla L_{t+1}(w_{t+1/2}) \|^2 \right] \nonumber \\
    & \leq \mathbb{E} \left[ L(w_t) \right] - \eta \mathbb{E} \left[ \langle \nabla L(w_{t+1/2}), \nabla L(w_t) \rangle \right] + \eta^2 \beta \mathbb{E} \left[ \| \nabla L_{t+1}(w_{t+1/2}) - \nabla L(w_{t+1/2}) \|^2 \right] + \eta^2 \beta \mathbb{E} \left[ \| \nabla L(w_{t+1/2}) \|^2 \right] \label{eq:jensen} \\
    & \leq \mathbb{E} \left[ L(w_t) \right] - \eta \mathbb{E} \left[ \langle \nabla L(w_{t+1/2}), \nabla L(w_t) \rangle \right] + \eta^2 \beta \frac{\sigma^2}{b} + \eta^2 \beta \mathbb{E} \left[ \| \nabla L(w_{t+1/2}) \|^2 \right] \nonumber \\
    & = \mathbb{E} \left[ L(w_t) \right] - \eta \mathbb{E} \left[ \langle \nabla L(w_{t+1/2}), \nabla L(w_t) \rangle \right] + \eta^2 \beta \frac{\sigma^2}{b} \nonumber \\
    & \hspace{1cm} - \eta^2 \beta \mathbb{E} \left[ \| \nabla L(w_t) \|^2 \right] + \eta^2 \beta \mathbb{E} \left[ \| \nabla L(w_{t+1/2}) - \nabla L(w_t) \|^2 \right] + 2\eta^2 \beta \langle \nabla L(w_{t+1/2}), \nabla L(w_t) \rangle \nonumber \\
    & \leq \mathbb{E} \left[ L(w_t) \right] - \eta \mathbb{E} \left[ \langle \nabla L(w_{t+1/2}), \nabla L(w_t) \rangle \right] + \eta^2 \beta \frac{\sigma^2}{b} \nonumber \\
    & \hspace{1cm} - \eta^2 \beta \mathbb{E} \left[ \| \nabla L(w_t) \|^2 \right] + \eta^2 \beta^3 \mathbb{E} \left[ \| w_{t+1/2} - w_t \|^2 \right] + 2\eta^2 \beta \langle \nabla L(w_{t+1/2}), \nabla L(w_t) \rangle \label {eq:smooth} \\
    & = \mathbb{E} \left[ L(w_t) \right] - \eta \mathbb{E} \left[ \langle \nabla L(w_{t+1/2}), \nabla L(w_t) \rangle \right] + \eta^2 \beta \frac{\sigma^2}{b} \nonumber \\
    & \hspace{1cm} - \eta^2 \beta \mathbb{E} \left[ \| \nabla L(w_t) \|^2 \right] + \eta^2 \beta^3 r^2 \mathbb{E} \left[ \| \nabla_{t+1}L(w_{t}) \|^2 \right] + 2\eta^2 \beta \langle \nabla L(w_{t+1/2}), \nabla L(w_t) \rangle \nonumber \\  & = \mathbb{E} \left[ L(w_t) \right] - \eta \mathbb{E} \left[ \langle \nabla L(w_{t+1/2}), \nabla L(w_t) \rangle \right] + \eta^2 \beta \frac{\sigma^2}{b} \nonumber \\
    & \hspace{1cm} - \eta^2 \beta \mathbb{E} \left[ \| \nabla L(w_t) \|^2 \right] + 2 \eta^2 \beta^3 r^2 \mathbb{E} \left[ \| \nabla L(w_{t}) \|^2 \right] + 2 \eta^2 \beta^3 r^2 \frac{\sigma^2}{b} + 2\eta^2 \beta \langle \nabla L(w_{t+1/2}), \nabla L(w_t) \rangle \nonumber
\end{align}
where (\ref{eq:jensen}) is based on Jensen's inequality and (\ref{eq:smooth}) is based on Assumption 2.
Then, by rearranging the terms, we have
\begin{align}
    \mathbb{E}\left[ L(w_{t+1}) \right] & \leq \mathbb{E} \left[ L(w_t) \right] - \eta^2 \beta (1 - 2\beta^2 r^2) \mathbb{E} \left[ \| \nabla L(w_t) \|^2 \right] + \eta^2 \beta (1 + 2 \beta^2 r^2) \frac{\sigma^2}{b} - \eta (1 - 2\eta \beta) \mathbb{E} \left[ \langle \nabla L(w_{t+1/2}), \nabla L(w_t) \rangle \right] \nonumber \\
    & \leq \mathbb{E} \left[ L(w_t) \right] - \eta^2 \beta (1 - 2\beta^2 r^2) \mathbb{E} \left[ \| \nabla L(w_t) \|^2 \right] + \eta^2 \beta (1 + 2 \beta^2 r^2) \frac{\sigma^2}{b} - \eta (1 - 2\eta \beta) \mathbb{E} \left[ \left(\frac{1}{2} - r\beta \kappa \right) \| \nabla L(w_t) \|^2 - \beta^2 r^2 \frac{\sigma^2}{2b} \right] \label{eq:uselemma3} \\
    & = \mathbb{E} \left[ L(w_t) \right] + \left( - \frac{\eta}{2} + \eta \beta r\kappa - 2\eta^2 \beta^2 r\kappa + 2 \eta^2 \beta^3 r^2 \right)  \mathbb{E} \left[ \| \nabla L(w_t) \|^2 \right] + \left(\eta^2 \beta + \eta \beta^2 r^2 \right) \frac{\sigma^2}{b} \nonumber \\
    & = \mathbb{E} \left[ L(w_t) \right] -\frac{\eta}{2} \left( 1 - 2\beta r \left( \kappa - 2\eta \beta \left( \kappa - \beta r \right) \right) \right)  \mathbb{E} \left[ \| \nabla L(w_t) \|^2 \right] + \left(\eta^2 \beta + \eta \beta^2 r^2 \right) \frac{\sigma^2}{b} \nonumber \\
    & \leq \mathbb{E} \left[ L(w_t) \right] -\frac{\eta}{4} \mathbb{E} \left[ \| \nabla L(w_t) \|^2 \right] + \left(\eta^2 \beta + \eta \beta^2 r^2 \right) \frac{\sigma^2}{b} \nonumber
\end{align}
where (\ref{eq:uselemma3}) is based on Lemma \ref{lemma:dot}.
The final inequality holds if $\eta \leq \frac{1}{2\beta}$ and $r \leq \frac{1}{2\beta}$ regardless of the value of $\kappa$.
\end{proof}
\end{lemma}

\begin{theorem}
    Assume the $\beta$-smooth loss function and the bounded gradient variance. Then, if $\eta \leq \frac{1}{2\beta}$ and $r \leq \frac{1}{2\beta}$, mini-batch SGD satisfies:
\begin{align}
    \frac{1}{T} \sum_{t=0}^{T-1} \mathbb{E} \left[ \| \nabla L(w_t) \|^2 \right] \leq \frac{4}{T\eta} \left(L(w_0) - \mathbb{E}\left[ L(w_T) \right] \right) + 4\left(\eta \beta + \beta^2 r^2 \right) \frac{\sigma^2}{b}.
\end{align}
\begin{proof}
Based on Lemma \ref{lemma:frame}, by averaging (\ref{eq:lemma4}) across $T$ iterates, we have
\begin{align}
    \frac{1}{T} \sum_{t=0}^{T-1} \mathbb{E}\left[ L(w_{t+1}) \right] &\leq \frac{1}{T}\sum_{t=0}^{T-1} \left( \mathbb{E} \left[ L(w_t) \right] -\frac{\eta}{4} \mathbb{E} \left[ \| \nabla L(w_t) \|^2 \right] + \left(\eta^2 \beta + \eta \beta^2 r^2 \right) \frac{\sigma^2}{b} \right) \nonumber
\end{align}
Then, we can have a telescoping sum by rearranging the terms as follows.
\begin{align}
    \frac{\eta}{4T} \sum_{t=0}^{T-1} \mathbb{E} \left[ \| \nabla L(w_t) \|^2 \right] & \leq \frac{1}{T} \sum_{t=0}^{T-1} \left( \mathbb{E}\left[ L(w_t) \right] - \mathbb{E}\left[ L(w_{t+1}) \right] \right) + \left(\eta^2 \beta + \eta \beta^2 r^2 \right) \frac{\sigma^2}{b} \nonumber \\
    & = \frac{1}{T} \left( L(w_0) - \mathbb{E}\left[ L(w_{T}) \right] \right) + \left(\eta^2 \beta + \eta \beta^2 r^2 \right) \frac{\sigma^2}{b} \nonumber
\end{align}
Finally, by dividing both sides by $\frac{\eta}{4}$, we have 
\begin{align}
    \frac{1}{T} \sum_{t=0}^{T-1} \mathbb{E} \left[ \| \nabla L(w_t) \|^2 \right] \leq \frac{4}{T\eta} \left(L(w_0) - \mathbb{E}\left[ L(w_T) \right] \right) + 4\left(\eta \beta + \beta^2 r^2 \right) \frac{\sigma^2}{b}.
\end{align}
\end{proof}
\end{theorem}

\subsection {Layer-Wise Loss Landscape Visulization}
We employed the visualization algorithm proposed in \cite{li2018visualizing}.
To obtain plots shown in Figure \ref{fig:landscape}, we conducted the following steps.
\begin{enumerate}
    \item Create one random vector that has the same size as the target layer.
    \item Create another vector that is orthogonal to the first vector.
    \item Divide them by their norms to make them have a norm of 1.
    \item Multiply $\alpha$ to the first vector and $\beta$ to the second vector and add them to the target layer parameters.
    \item Collect the training loss using the perturbed model.
    \item Repeat steps 4 and 5 using $\alpha \in \{-20, \cdots, 19 \}$ and $\beta \in \{-20, \cdots, 19\}$.
\end{enumerate}
These steps provide us with the approximated 3-D loss landscape figures shown in Figure \ref{fig:landscape}.
We used the same 1024 training images to calculate the loss value at all $1,600$ grid points.
The z-axis range is fixed from 0 to 4.5 for all six layers.
Because the model is perturbed only at a single target layer, the landscape figures provide insights into how all individual layers affect the model's generalization performance.
